\documentclass[10pt]{article} 

\usepackage[preprint]{rlj}           

\usepackage{amssymb}            
\usepackage{mathtools}          
\usepackage{mathrsfs}           
\usepackage{graphicx}           
\usepackage{subcaption}         
\usepackage[space]{grffile}     
\usepackage{url}                
\usepackage{lipsum}             

\usepackage{amsmath}
\usepackage{amssymb}
\usepackage{mathtools}
\usepackage{amsthm}

\usepackage{algorithm}
\usepackage{algorithmic}
\newcommand{\spanv}[1]{\textrm{sp}(#1)}
\newcommand{\adv}[2]{\textrm{adv}(#1,#2)}
\DeclareMathOperator*{\argmax}{arg\,max}
\DeclareMathOperator{\E}{\mathbb{E}}
\DeclareMathOperator{\st}{\mathrm{st}}
\usepackage[T1]{fontenc} 
\usepackage{lmodern}

\theoremstyle{plain}
\newtheorem{theorem}{Theorem}[section]

\newtheorem{lemma}[theorem]{Lemma}
\newtheorem{corollary}[theorem]{Corollary}
\theoremstyle{definition}
\newtheorem{definition}[theorem]{Definition}
\newtheorem{assumption}[theorem]{Assumption}
\theoremstyle{remark}

\title{Geometric Re-Analysis of Classical MDP Solving Algorithms}

\setrunningtitle{Geometric Re-Analysis of Classical MDP Solving Algorithms}

\author{Arsenii Mustafin\textsuperscript{1, $^\dagger$},
Aleksei Pakharev\textsuperscript{2, $^\dagger$},
Alex Olshevsky\textsuperscript{3}, 
Ioannis Ch. Paschalidis\textsuperscript{3}}

\emails{\{aam,alexols,yannisp\}@bu.edu, \ pakhara@mskcc.org }

\affiliations{
$^{1}$\textbf{Department of CS, Boston University}\\
$^{2}$\textbf{Memorial Sloan Kettering Cancer Center}\\
$^{3}$\textbf{Department of ECE, Boston University}
\par 
$^\dagger$ Equal contribution
}

\contribution{
    We develop a new geometry-based framework for analyzing the convergence of VI and PI algorithms. In particular, we identify a rotation component in the VI algorithm and introduce an MDP transformation that modifies the discount factor. 
    }
    {
    We extend the framework from \cite{mdp_geometry}. 
    }

\contribution{
    Using the discount factor transformation, we show that the theoretical convergence of both VI and PI can be improved when the transformation can be applied in a way that preserves the regularity of the MDP.
    }
    {
    None
    }

\contribution{
     We show that in the case of a 2-state MDP, the number of iterations required by PI to reach the optimal policy is upper bounded by the number of actions.
    }
    {
    In this paper we call actions what is usually called state-action pairs. Previously convergence bounds were dependent on the discount factor and it does not affect our results.
    }

\contribution{
     For Value Iteration, we show that it benefits from information exchange between states, leading to a convergence rate faster than $\gamma$ when the Markov reward process (MRP) induced by the optimal policy is strongly connected. In this case we improve the total number of iterations from:
    \begin{equation*}
        \mathcal{O}\left( \frac{\log{1/\epsilon} + \log (1/(1-\gamma)) }{\log{(1/\gamma)}} \right) \quad\text{to}\quad
        \mathcal{O}\left(\frac{\log (1/\epsilon) + \log(1/(1-\gamma)) }{\log (1/\gamma) + 
        \textcolor{blue}{\log(1/\tau^{1/N})}} \right),
    \end{equation*}
    where $\tau^{1/N}$ is a  measure of the mixing rate associated with an optimal policy. 
        }
    {
    None
    }


\summary{We extend a recently introduced geometric interpretation of Markov Decision Processes (MDPs) that provides a new perspective on MDP algorithms and their dynamics. Based on this view, we develop a novel analytical framework that simplifies the proofs of existing results and enables us to derive new ones.

Specifically, we analyze the behavior of two classical MDP-solving algorithms: Policy Iteration (PI) and Value Iteration (VI). For each algorithm, we first describe its dynamics in geometric terms and then present an analysis along with several convergence results. We begin by introducing an MDP transformation that modifies the discount factor $\gamma$ and demonstrate how this transformation improves the convergence properties of both algorithms, provided that it can be applied such that the resulting system remains a regular MDP. Second, we present a new analysis of PI in a 2-state MDP case, showing that the number of iterations required for convergence is bounded by the number of state-action pairs.  Finally, we reveal an additional convergence factor in the VI algorithm for cases with a connected optimal policy, which is attributed to an extra rotation component in the VI dynamics.
}

\begin{document}

\makeCover  
\maketitle  

\begin{abstract}
We build on a recently introduced geometric interpretation of Markov Decision Processes (MDPs) to analyze classical MDP-solving algorithms: Value Iteration (VI) and Policy Iteration (PI). First, we develop a geometry-based analytical apparatus, including a transformation that modifies the discount factor $\gamma$, to improve convergence guarantees for these algorithms in several settings. In particular, one of our results identifies a rotation component in the VI method, and as a consequence shows that  when a Markov Reward Process (MRP) induced by the optimal policy is irreducible and aperiodic, the asymptotic convergence rate of value iteration is strictly smaller than $\gamma$. 

\end{abstract}

\section{Introduction}
\subsection{History of the Subject and Previous works}
A Markov Decision Process (MDP) is a widely used mathematical framework for sequential decision-making. It was first introduced in the late 1950s, along with foundational algorithms such as Value Iteration (VI) \citep{bellmandp} and Policy Iteration (PI) \citep{howard1960dynamic}. These algorithms have become the foundation for various theoretical and practical methods for solving MDPs, which today form the backbone of applied Reinforcement Learning (RL).

Over the following decades, significant advancements were made, culminating in a comprehensive summary of key results by Puterman in 1990 \citep{puterman1990markov}. In recent years, the growing popularity of practical RL algorithms has renewed interest in MDP analysis, leading to several notable developments.

For Value Iteration, \citet{howard1960dynamic} showed that the algorithm's convergence rate is upper bounded by the discount factor $\gamma$ and that this upper bound is achievable. However, in most practical cases, VI exhibits faster convergence. Subsequent works focused on analyzing this convergence and providing guarantees for MDP instances under additional assumptions \citep{puterman1990markov, feinberg2014value}.

For Policy Iteration, a significant gap in understanding its convergence properties remains. Important progress was made by \citet{ye2011simplex, hansen2013strategy, scherrer2013improved}, where the authors significantly improved the upper bound on the number of iterations required for convergence in terms of $1-\gamma$, where $\gamma$ is the MDP discount factor. At the same time, a separate line of work \citep{fearnley2010exponential, hollanders2012complexity, hollanders2016improved} showed that the complexity of PI can be exponential when the discount factor $\gamma$ is not fixed.

In this paper, we extend the analysis of VI and PI by leveraging the recently introduced geometric interpretation of MDPs \citep{mdp_geometry}. In their work, the authors proposed viewing MDPs from a geometric perspective, drawing analogies between common MDP problems and geometric problems. We build on this approach to develop new analytical methods for studying Value Iteration and Policy Iteration. These tools allow us to simplify the analysis and improve convergence results in several cases.

\subsection{Motivation and Contribution}

The primary motivation for this paper is to address existing gaps in the understanding and analysis of fundamental MDP algorithms. In the case of Value Iteration, the gap lies between the convergence rate observed in most settings and the theoretically guaranteed convergence rate. For Policy Iteration, the gap exists between its upper and lower convergence bounds. The geometry-based analysis proposed in this work enhances our understanding of algorithm dynamics and has the potential to guide the design of new algorithms.

Our main contributions are as follows: 
\begin{itemize} 
\item We develop a new geometry-based framework for analyzing the convergence of Value Iteration and Policy Iteration. In particular, we identify a rotation component in the Value Iteration algorithm and introduce an MDP transformation that modifies the discount factor $\gamma$. 
\item Using the discount factor transformation, we show that the theoretical convergence of both VI and PI can be improved when the transformation can be applied in a way that preserves the regularity of the MDP.
\item We show that in the case of a 2-state MDP, the number of iterations required by PI to reach the optimal policy is upper bounded by the number of actions\footnote{For us, actions are unique to a state, so using earlier terminology (which we do not use in this paper), the claim is that the number of iterations is upper bounded by the number of state-action pairs.}.
\item For Value Iteration, we show that it benefits from information exchange between states, leading to a convergence rate faster than $\gamma$ when the Markov reward process (MRP) induced by the optimal policy is strongly connected. In this case we improve the total number of iterations from:
\begin{equation*}
    \mathcal{O}\left( \frac{\log{1/\epsilon} + \log (1/(1-\gamma)) }{\log{(1/\gamma)}} \right) \quad\text{to}\quad
    \mathcal{O}\left(\frac{\log (1/\epsilon) + \log(1/(1-\gamma)) }{\log (1/\gamma) + 
    \textcolor{blue}{\log(1/\tau^{1/N})}} \right),
\end{equation*}
\end{itemize} where $\tau^{1/N}$ is a  measure of the mixing rate associated with an optimal policy. While the former convergence rate on the left blows up polynomially as $\gamma \rightarrow 1$ (due to the $\log (1/\gamma)$ in the denominator which approaches zero as $(\gamma - 1)/\gamma$), the new convergence on the right rate blows up {\em logarithmically} as $\gamma \rightarrow 1$.

Additionally, we give simplified geometry-based proofs for a several established facts.

\section{Mathematical setting} \label{sec:math_background}
\subsection{Basic MDP setting}

We employ an MDP framework from \cite{mdp_geometry}. An MDP is defined by the tuple $\mathcal{M} = \langle \mathcal{S}, \mathcal{A}, \mathrm{st}, \mathcal{P}, \mathcal{R}, \gamma \rangle$, where $\mathcal{S} = \{s_1, \ldots, s_n\}$ represents a finite set of $n$ states, and $\mathcal{A}$ is a finite set containing $m$ possible actions, where each action is defined in a unique state. Therefore, actions in the framework correspond to state-action pairs in earlier literature. This relation is defined by a mapping $\mathrm{st}$, which maps actions $a$ to the state where it can be chosen. Each action $a$ is characterized by the probability distribution $\mathcal{P}(a) = (p^a_1, \dots, p^a_n)$ and deterministic rewards $r^a$ which are described by $\mathcal{R}: \mathcal{A} \to \mathbb{R}$.

An agent interacting with the MDP follows a policy, which is a map $\pi: \mathcal{S} \rightarrow \mathcal{A}$ that satisfies $\mathrm{st}(\pi(s)) = s$ for all $s \in \mathcal{S}$. We consider deterministic stationary policies, where a single action is chosen for each state throughout the trajectory. If a policy $\pi$ chooses action $a$ in state $s$, we write $a \in \pi$. Therefore, a policy $\pi$ can be described as the set of its actions, $\pi = \{a_1, \dots, a_n\}$.

The value of a policy $\pi$ at state $s$, denoted $V^\pi(s)$, is the expected discounted reward over infinite trajectories starting from $s$ under policy $\pi$: 
\begin{equation*} V^\pi(s) = \E \left[ \sum_{t=1}^{\infty} \gamma^t r_t \right],
\end{equation*}
where $r_t$ is the reward at time $t$. The value vector $V^\pi$ uniquely satisfies the Bellman equation $T^\pi V^\pi = V^\pi$, where $T^\pi$ is the Bellman operator: 
\[ (T^\pi V)(s) = r_\pi + \sum_{s'} P\left(s'\,|\,\pi(s)\right) \gamma V(s'). \]
Evaluating the values of a given policy $\pi$ is known as the Policy Evaluation problem. The main challenge, however, is to identify the \textbf{optimal} policy $\pi^*$ that satisfies: 
\[ V^{\pi^*} (s) \ge V^\pi(s), \quad \forall \pi, s. \]
This policy can be found by the Policy Iteration algorithm, but each iteration of it requires inverting an $n \times n$ matrix. Alternatively, we may aim to find an approximate solution, \textbf{$\epsilon$-optimal} policy $\pi^\epsilon$, which satisfies: 
\[ V^{\pi^*} (s) - V^{\pi^\epsilon}(s) < \epsilon, \quad \forall s. \]
An $\epsilon$-optimal policy can be found using the Value Iteration (VI) algorithm. The required number of iterations depends on $\epsilon$, with each iteration having a computational complexity of $\mathcal{O}(nm)$.

MDP setting presented above was reinterpreted in geometric terms in \cite{mdp_geometry}, and our analyzis relies on this geometric interpretation. In this work the authors suggest to view MDP actions as vectors in a linear space called \textbf{action space}. To construct the $(n+1)$-dimensional action vector $a^+$ from an action $a$, one needs to write action reward as the first entry of the vector, which the authors refer to as the $0$-th coordinate --- $c_0 = r^a$. Then, the next $n$ entries of the action vector are equal to $c^a_i = \gamma p^a_i$, except the $s$-th entry, where $s = \st(a)$ is the state of $a$. This entry is modified to be $c^a_s = \gamma p^a_s -1$. Therefore, the sum of $n$ last entries of any action vector is equal to $\gamma -1$, while the entry corresponding to the action's state is negative and all other entries are positive.

Then the authors define policy vectors. For each policy $\pi$ the vector $V^\pi_+ = (1, V^\pi(1), \dots, V^\pi(n))^T$ is composed of the policy values in all states and an extra bias coordinate with the value $1$. Then in the geometric sense a policy $\pi$ can be represented as the hyperplane $\mathcal{H}^\pi$ of all vectors orthogonal to $V^\pi_+$. Note that such a hyperplane can be constructed for any vector of values $V_+$, \textit{i.e.} we do not need an actual policy --- action choice rule --- to construct such hyperplane. The hyperplanes which are produced by set of values without actual actions behind them are called \textbf{pseudo-policies}.

For any action vector $a^+$ and policy vector $V^\pi_+$ the inner product $a^+V^\pi_+$ is equal to the advantage of action $a$ with respect to policy $\pi$, the key quantity in an MDP. It was shown in \cite{mdp_geometry} that the dynamics of the VI and PI algorithms are determined by advantages. It is also shown that the transformation procedure $\mathcal{L}^\delta_s$, which shifts all policy values at $s$ by $\delta$, preserves advantages. Additionally, the transformation preserves the stopping and filtering criteria that we use. As a result, it maintains the dynamics of the PI and VI algorithms. This implies that for any MDP $\mathcal{M}$, we can consider its \textbf{normalization $\mathcal{M}^*$}, the MDP obtained from $\mathcal{M}$ by a series of $\mathcal{L}$ transformations where all values of its optimal policy are equal to $0$. It follows that the PI and VI algorithms will exhibit the same dynamics on $\mathcal{M}$ and $\mathcal{M}^*$.

\textbf{Note:} In this paper we carry out the analysis on normalized MDPs, $\mathcal{M} = \mathcal{M}^*$. In particular, it implies that $r^a=0 \,  \forall a \in \pi^*$, $r^b <0 \, \forall b \notin \pi^*$.



\subsection{Algorithms} \label{ssec:algorithms}

\begin{algorithm}[t]
   \caption{Value Iteration Algorithm}
   \label{alg:value_iter}
\begin{algorithmic}
   \STATE {\bfseries Parameters} Learning rate $\alpha$, desired accuracy $\epsilon$, stopping criterion $H(\mathcal{I})$ and action filtering rule $F(\cdot| \mathcal{I})$.
   \STATE {\bfseries Initialize} $V_0$, set $t=0$ and $\mathcal{A}_0 = \mathcal{A}$.
   \STATE {\bfseries Iteration} Select $\mathcal{S}_t$
   \STATE Compute $U = \max_{a \in \mathcal{A}_t} r^a + \gamma \sum_{i=1}^n p^a_i V_t(i)$,
   \STATE Compute $V_{t+1} (s) =
   \begin{cases}
    (1-\alpha) V_{t} + \alpha U &\text{ if } s \in \mathcal{S}_t,\\
    V_t(s) \quad \quad &\text{ if } s \notin \mathcal{S}_t.       
   \end{cases}$
   
   \STATE Apply filtering $\mathcal{A}_{t+1} = F(\mathcal{A}_t|\mathcal{I}_t)$
   \IF{not $H(\mathcal{I}_t)$}
    \STATE Increment $t$ by 1 and return to the Iteration step.
   \ELSE
    \STATE Output $\pi:\,\pi(s) = \argmax_{a \in \mathcal{A}_t}  r(s,a) + \gamma P_a(s)V_{t}$.
    \ENDIF
\end{algorithmic}
\end{algorithm}

The Value Iteration algorithm (Algorithm \ref{alg:value_iter}) is presented in a non-standard, more general form, which allows for multiple versions of it to be discussed. In this algorithm:
\begin{itemize}
\item $\mathcal{I}$ is all information which reflects the overall state of the algorithm, which might include quantities one wants to track during the run of the algorithm. In particular, $\mathcal{I}_t$ is the information available after $t$ iterations of the algorithm.
\item $H(\mathcal{I})$ is the stopping criterion, logical function of the system information, the output of which is \textit{true} when algorithm execution should be stopped. The choices we consider are: time-based $H(\mathcal{I}_t): t=T_{\rm max}$ maximum number of iterations reached; span-based $H(\mathcal{I}_t): \spanv{V_t - V_{t-1}} \le \epsilon(1-\gamma)/\gamma$, which allows to obtain $\epsilon$-optimal policy; action-based $H(\mathcal{I}_t): |\mathcal{A}_t|=n$ is used when action filtering is applied, it allows to obtain true optimal policy.
\item $F(\cdot| \mathcal{I})$ is the action filtering function. We add it to reflect certain criteria shrinking the pool of possible actions participating in the optimal policy as $t$ increases.
This technique also allows to stop at the exact solution under certain assumptions.
\item $\mathcal{S}_t$ are the states to be updated during the iteration $t$. If $\mathcal{S}_t = \mathcal{S}$ for all $t$, we call the update synchronous, otherwise the update is called asynchronous.
\item $\alpha$ is the learning rate of the algorithm. We only consider algorithms with constant learning rate.
\end{itemize}
We call a version of Value Iteration with synchronous update, learning rate $\alpha=1$ and without action filtering \textit{standard}.

\begin{algorithm}[t]
   \caption{Policy Iteration Algorithm}
   \label{alg:pol_iter}
\begin{algorithmic}
   \STATE {\bfseries Parameters} None.
   \STATE {\bfseries Initialize} $\pi_0$, set $t=0$.
   \STATE {\bfseries Iteration} Construct $P_t$ and $r_t$ from $\pi_t$ 
   \STATE \quad \textbf{Policy Evaluation:} Compute $V_{t+1} = (I - \gamma P_t)^{-1}r_t$,
   \STATE \quad \textbf{Policy Improvement:} Construct $\pi_{t+1}: \pi_{t+1}(s) = \argmax \adv{a}{V_{t+1}}$.
   \IF{$\pi_{t+1} \ne \pi_t$}
    \STATE Increment $t$ by 1 and return to the Iteration step.
   \ELSE
    \STATE Output $\pi_t$.
    \ENDIF
\end{algorithmic}
\end{algorithm}

As for the Policy Iteration algorithm (PI), in this paper we consider the standard version of it from \cite{howard1960dynamic} (Algorithm \ref{alg:pol_iter}), which updates actions in all states simultaneously (Howard PI).

\section{Transformation of the Discount Factor $\gamma$}

\begin{figure*}[ht] 
\begin{center}
\includegraphics[width=\textwidth]{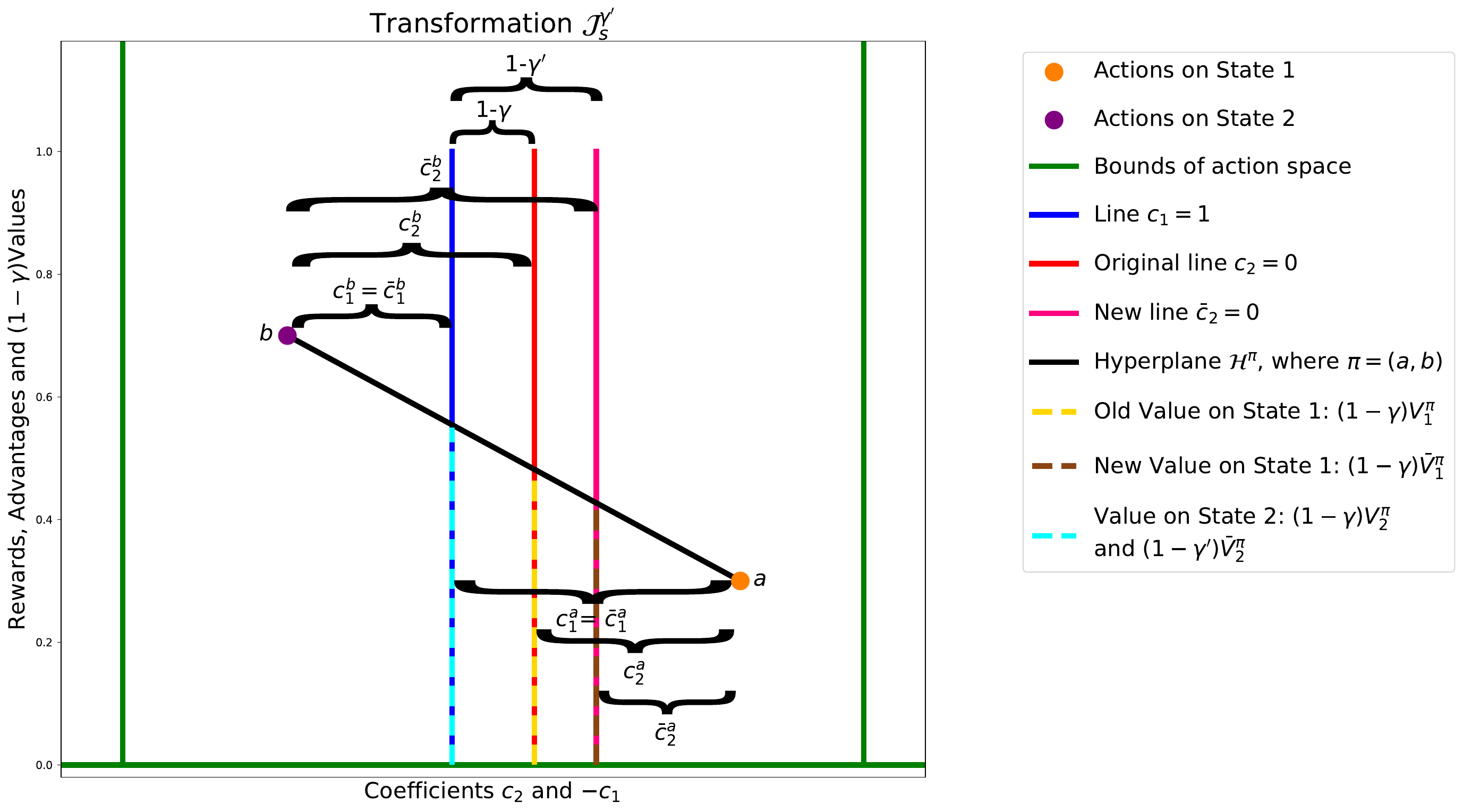}
\end{center}
\caption{Illustration of a transformation $\mathcal{J}_s^{\gamma'}$ in the case of 2-state MDP, where $s=2$ and the discount factor is updated from $\gamma$ to $\gamma'$. Dots $a$ and $b$ on the plot represent actions of the MDP on the states $1$ and $2$ resp., the $x$-axis is equal to $c_1 = \gamma - 1 - c_2$, and the $y$-axis is equal to the reward of an action. Blue and teal lines lie on $c_1 = \bar{c}_1 = 0$, red and yellow lines lie on $c_2 = 0$, and purple and brown lines lie on $\bar{c}_2 = 0$. The distance between blue and magenta lines is $1-\gamma$, and the distance between blue and red lines is $1-\gamma'$.
After the transformation, action coefficients related to state 1 remain unchanged ($c^a_1 = \bar{c}^a_1$, $c^b_1 = \bar{c}^b_1$) while those related to state 2 change by $\gamma' - \gamma$. The value on state 2 is equal to the length of the cyan bar divided by $1-\gamma$ before the transformation, and divided by $1-\gamma'$ after the transformation. The value on state 1 is more easily accessed using the value on state 2 as reference. Before the transformation, $V_1^\pi - V_2^\pi$ is equal to the difference of cyan and brown bars divided by $1-\gamma$, while $\bar{V}_1 - \bar{V}_2$ is equal to the difference of cyan and yellow bars divided by $1-\gamma'$.  This implies that the actual difference does not change: $V_1^\pi - V_2^\pi = \bar{V}_1^\pi - \bar{V}_2^\pi$.}
\label{fig:J_transformation}
\end{figure*}

In this section we define an MPD transformation $\mathcal{J}_s^{\gamma'}$ that changes the discount factor from $\gamma$ to $\gamma'$. Geometrically, it does not move the action vector or policy hyperplanes, but changes the coordinate $c_s$ corresponding to state $s$ by moving the corresponding zero level vertical hyperplane, consequently changing the coefficients $c_s$ and values on all states (Figure \ref{fig:J_transformation}). Denote $\gamma'$ as the new discount factor. Then the transformation rule for every action $a$ and policy $\pi$ is as follows:

\begin{itemize}
    \item The reward $r^a$ remains unchanged.
    \item Every coefficient $c_i^a, i \ne s$ remains unchanged.
    \item The new coefficient $\bar{c}_s^a$ corresponding to state $s$ changes to $\bar{c}_s^a := c_s^a - (\gamma-\gamma')$.
    \item The new value $\bar{V}^\pi(s)$ of every policy on state $s$ is set to $\bar{V}^\pi(s) = V^\pi(s) \frac{1-\gamma}{1-\gamma'}$.
    \item The value $\bar{V}^\pi (i)$ of every policy on every other state $i \ne s$ is being transformed such that the value differences are preserved: $\bar{V}^\pi (i) = V^\pi(i) + (\bar{V}^\pi (s) - V^\pi(s))$.
\end{itemize}

The key property of the transformation $\mathcal{J}_s^{\gamma'}$ is that it, similarly to transformation $\mathcal{L}_s^\delta$, preserves the key quantities characterizing MDP dynamics, which is stated in the following theorem.

\begin{theorem} \label{thm:J_preserve}
Transformation $\mathcal{J}_s^{\gamma'}$ preserves $(1)$ advantage $\adv{a}{\pi}$ of any action $a$ with respect to any policy $\pi$; $(2)$ preserves the vector span $\spanv{V^\pi}$, for any pseudo-policy $V^\pi$.
\end{theorem}
\begin{proof}
The proof is given in Appendix \ref{app:J_preserve_proof}.
\end{proof}

The significance of Theorem \ref{thm:J_preserve} is implied by the dependency of the convergence guarantees provided for the Value Iteration and Policy iteration algorithms on the discount factor $\gamma$. Therefore, if we can safely decrease $\gamma$, it gives an immediate yield in terms of a faster guaranteed convergence. In fact, to perform the transformation and decrease $\gamma$ safely, by which we mean that the coefficients constraint holds (only one of the coefficients is negative), we need that for some state $i$ all correspondent coefficients $c_i^a$ are positive, $\exists \,i: c_i^a > 0\,\forall a, \st{(a)}\ne i$. Then we can decrease $\gamma$ by $\min_a c_i^a, \st{(a)}\ne i$. It implies the following definition:

\begin{definition}
The \textbf{effective} value $\gamma_{\rm eff}$ of the discount factor of MDP $\mathcal{M}$ is the minimum possible value of $\gamma$ that can be obtained by applying safe transformations $\mathcal{J}_s^{\gamma'}$.
\end{definition}

With this definition we have two corollaries regarding the convergence of PI and VI algorithms

\begin{corollary}
The number of iterations $T_{PI}$ required for the Policy Iteration algorithm to output the optimal policy can be upper bounded by:
\begin{equation*}
    T_{PI} =\mathcal{O}\left(\frac{|A|}{1-\gamma_{\rm eff}} \right) \le 
    \mathcal{O}\left(\frac{|A|}{1-\gamma}\right) 
\end{equation*}
\end{corollary}
To the best of our knowledge, this result is novel.

\begin{corollary}
The number of iterations $T_{VI}$ required for the standard Value Iteration algorithm to converge to the $\epsilon$-optimal policy can be upper bounded by:
\begin{align*}
T_{VI}&=\mathcal{O}\left( \frac{\log{(1/\epsilon)} + \log (1/(1-\gamma_{\rm eff})) }{\log{(1/\gamma_{\rm eff})}} \right)  \le \mathcal{O}\left( \frac{\log{(1/\epsilon)} + \log (1/(1-\gamma)) }{\log{(1/\gamma)}} \right) \\
\end{align*}
\end{corollary}
This result might be seen as a slight improvement over Corollary 6.6.8 from \cite{puterman2014markov}  and Theorem 1 from \cite{feinberg2020complexity} since we do not require the actions on the same state to be considered. Additionally, the proof we give here is significantly simpler.

\section{Policy Iteration in Two-State MDP} \label{sec:PI}

\begin{figure*}[ht] 
\begin{center}
\includegraphics[width=\textwidth]{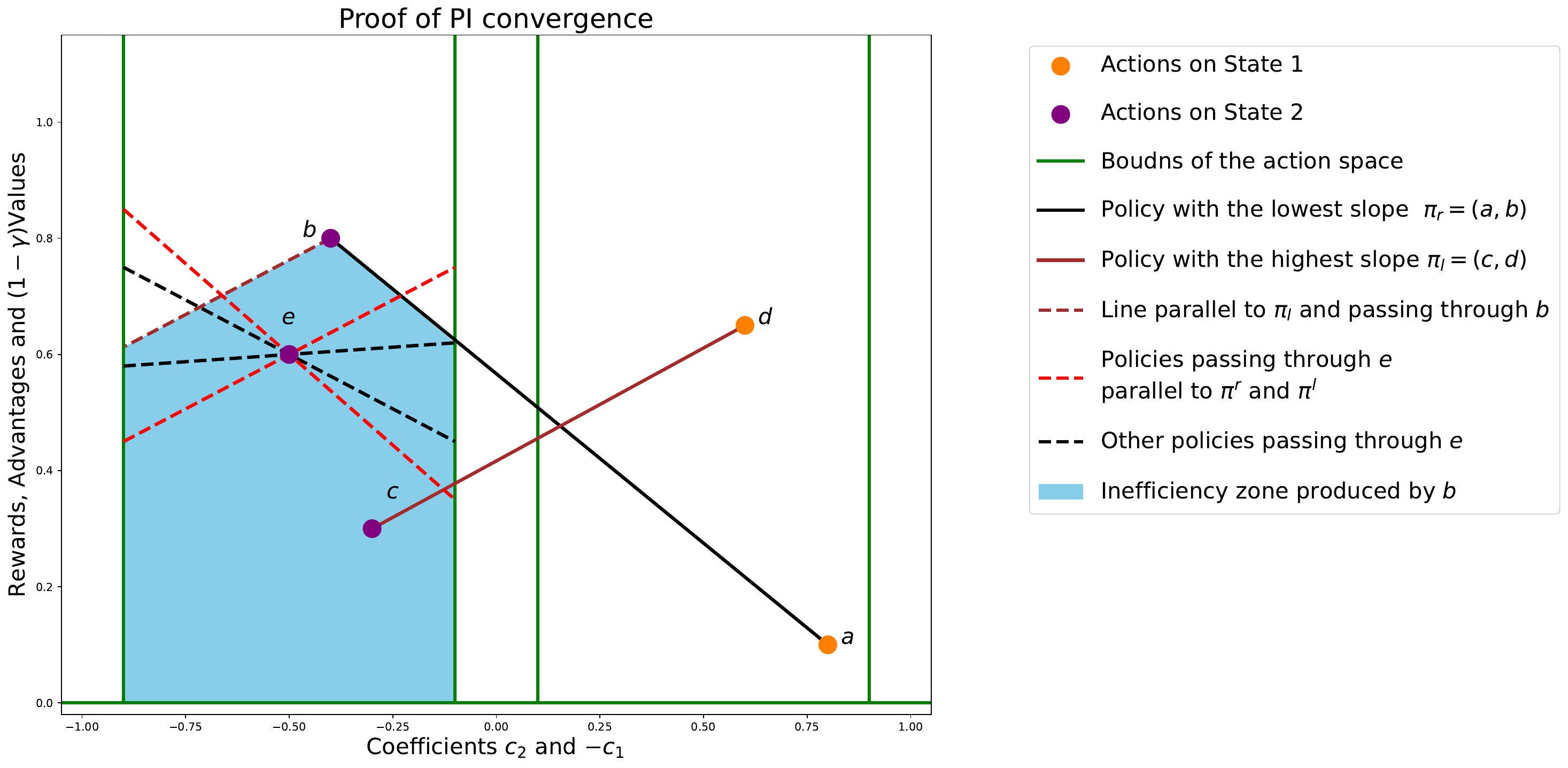}
\end{center}
\caption{Proof of Theorem \ref{thm:PI_2state}. For any set of actions $\mathcal{A}$ and the corresponding set of policies $\mathcal{U}$ formed by them, we identify the policies in $\mathcal{U}$ with the most extreme slopes. Denote the policy with the smallest slope as $\pi_r$ (formed by actions $a$ and $b$) and the policy with the largest slope as $\pi_l$ (formed by actions $c$ and $d$). If we draw two lines parallel to $\pi_l$ and $\pi_r$ through any action (for example, action $b$ as shown in the Figure), the area below both lines forms an \textbf{inefficiency zone}: any action $e$ within this zone is inefficient within $\mathcal{U}$ because action $b$ lies above any policy that passes through $e$.
Next, we choose a state where the vertical difference between $\pi_r$ and $\pi$ increases with the corresponding coefficient (State 2 in the Figure). The action that participates in the policy with the lower value at this state (action $c$) falls inside the inefficiency zone of the action that forms the policy with the higher value (action $b$).
}
\label{fig:pi_proof}
\end{figure*}

In this section we give a geometric proof of the fact that the number of iterations required by Policy Iteration algorithm to converge in 2-state MDP is upper bounded by the number of actions in it. The key idea of our analysis is to consider all possible trajectories of PI together.

 Denote the policy produced after $t$ steps of PI as $\pi_t$. Then, geometrically, PI dynamics is as follows: during the Policy Improvement step PI chooses an action $a^s_{t+1}$ which has the highest vertical distance to the hyperplane $\mathcal{H}^{\pi_t}$ at every state $s$. These actions are then collected in a policy $\pi_{t+1} = (a^1_{t+1}, \dots, a^n_{t+1})$. During the Policy Evaluation step PI constructs a hyperplane $\mathcal{H}^{\pi_{t+1}}$ which passes through them. To describe this relations, we say that actions $a^1_{t+1}, \dots, a^n_{t+1}$ are \textbf{produced} by the policy $\pi_{t}$, while they \textbf{form} the policy $\pi_{t+1}$. Both notions are expandable on sets of policies and actions. A set of policies $\mathcal{U}$ is formed by a set of actions $\mathcal{A}$ if $\mathcal{U}$ consists of all policies formed by actions in $\mathcal{A}$.
 Similarly, a set of actions $\mathcal{A}'$ is produced by $\mathcal{U}'$ if $\mathcal{A}'$ consists of all actions produced by policies in $\mathcal{U}'$. Note, that $\mathcal{A}'$ will always have at least one action at any state. We call actions in $\mathcal{A}'$ \textbf{efficient} on $\mathcal{U}'$.

This notation allows us to describe the global dynamics of PI in terms of these sets. We start with $\mathcal{A}_0 = \mathcal{A}$, all actions available in the MDP. Then, $\mathcal{A}_0$ forms the set of policies $\mathcal{U}_0$, which produces the set of actions $\mathcal{A}_1 \subset \mathcal{A}_0$ and so on. The following theorem establishes a key property of the PI dynamics.

\begin{theorem} \label{thm:PI_2state}
In a two-state MDP for any set of actions $\mathcal{A}$, $\left| \mathcal{A}\right| \ge 3$, with actions on both states, there is at least one action which is not efficient on the set of policies $\mathcal{U}$ formed by actions in $\mathcal{A}$. \end{theorem}
\begin{proof}
A geometric proof if presented on Figure \ref{fig:pi_proof} and an algebraic proof is presented in Appendix \ref{app:PI_2}. \end{proof}
\begin{corollary}
In a two-state MDP the number of iterations required by the Policy Iteration algorithm to converge is bounded by the number of actions in it.\end{corollary}
\begin{proof}
Theorem \ref{thm:PI_2state} states that for any $\mathcal{A}_t$ there is an action $a_t$ which is not efficient on $\mathcal{U}_t$. It implies that $a_t \notin \mathcal{A}_{t+1}$, which, in turn implies that $|\mathcal{A}_{t+1}| \le |\mathcal{A}_{t}| - 1$. Therefore, after at most $T \le m-n$ iterations $|\mathcal{A}_T| = n$, which implies that PI is guaranteed to output the optimal policy after $T$ iterations.
\end{proof}

\section{Analysis of Value Iteration} \label{sec:VI}
\subsection{Our Approach} \label{ssec:vi_approach}

\begin{figure*}[ht] 
\begin{center}
\includegraphics[width=\textwidth]{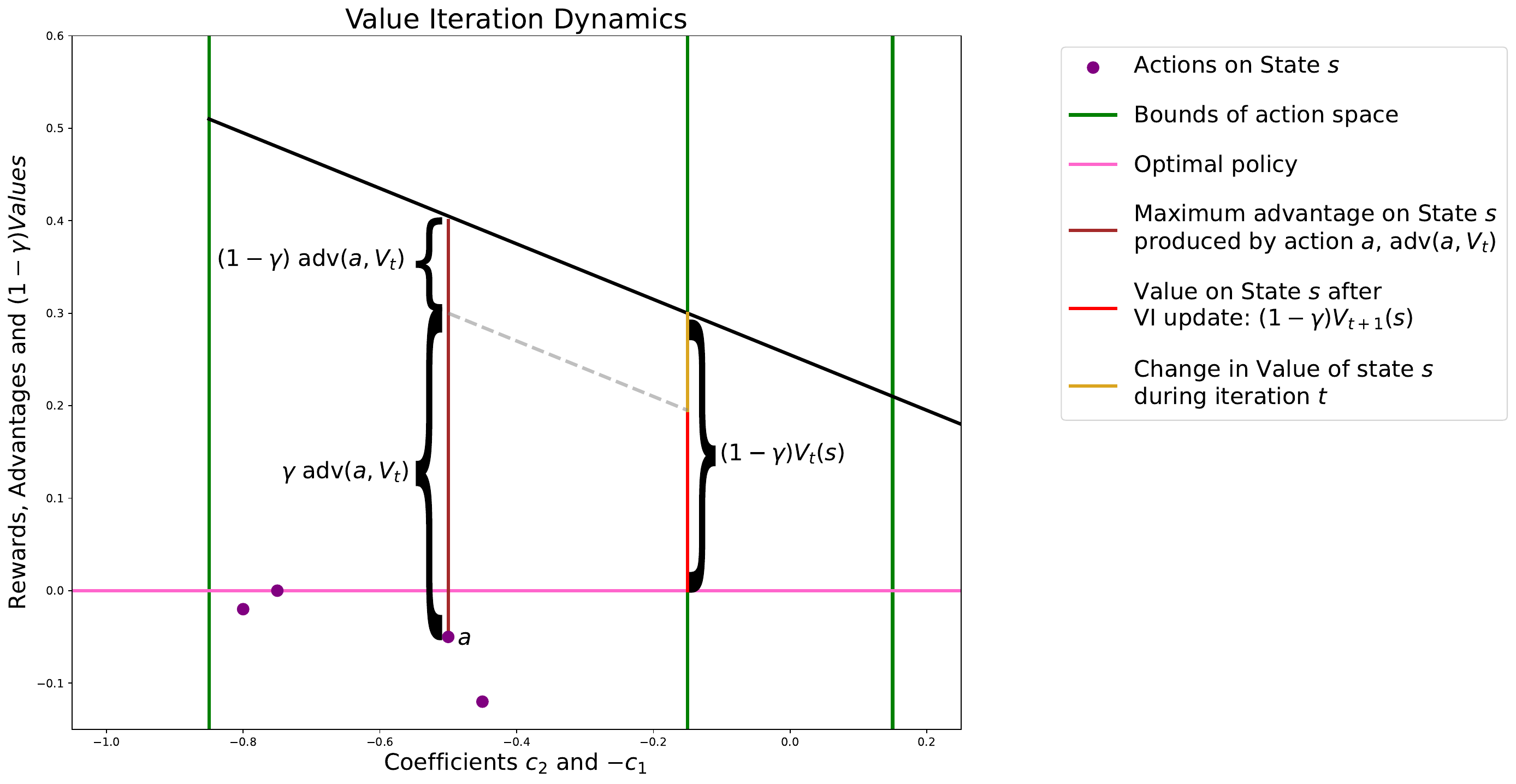}
\end{center}
\caption{Illustration of the Value Iteration algorithm dynamics: $V_{t+1}(s) = V_t(s) + \adv{a^*}{V_{t}}$ (figure adapted from \citet{mdp_geometry}). Graphically, VI can be interpreted as subtracting the length of the brown bar, scaled by $1 - \gamma$, from the value bar. The subtracted length is represented by the yellow bar, while the remaining value is shown as a red bar.
Assume that $s$ is the state with the maximum value $V_t(s)$, as depicted in the figure. For $V(s)$ to contract exactly by $\gamma$ (i.e., $V_{t+1}(s) = \gamma V_t(s)$), the optimal action must be chosen as the maximizer and must lie exactly on the self-loop line (or its projection in the multidimensional case). For the state with the minimum value, $s'$, the subtracted values will always be less than $(1 - \gamma) V_t(s')$, unless both conditions are met.
Together, these two facts explain the source of the extra convergence in the Value Iteration update: it skews the pseudo-policy $V_t$ toward horizontal hyperplane at a faster rate than it converges to zero.}
\label{fig:value_iteration_dynamics}
\end{figure*}

In this section, we analyze the convergence of the Value Iteration algorithm (Algorithm \ref{alg:value_iter}). We show that when VI is viewed through the dynamics of the pseudo-policy hyperplane produced by the values $V_t$ at each iteration, it does not only converge this hyperplane toward the optimal policy (which is $0$ in the case of a normalized MDP) but also skews it toward the horizontal slope. Under certain assumptions, the rate of this skewing is higher than the rate of convergence.

The key idea of the analysis is illustrated in Figure \ref{fig:value_iteration_dynamics}. It shows that in a 2-state MDP case, the Value Iteration algorithm subtracts more than $(1-\gamma)V_t(s)$ from the value of the state $s$ with the maximum value and less than $(1-\gamma)V_t(s')$ from the value of the state $s'$ with the minimum value. To observe the same effect in the multidimensional case, we require an assumption on the connectivity of the optimal policy, which is stated below. Additionally, we assume the uniqueness of the optimal policy.
\begin{assumption} \label{ass:irred_and_aper_MDP}
The MRP implied by the unique optimal policy $\pi^*$ is irreducible and aperiodic.
\end{assumption}

We are now going to show the faster-than-$\gamma$ convergence rate of the standard Value Iteration algorithm, which comes from two sources. One of them is the fact that optimal actions lie inside the dedicated area, which results in the mixing properties of the matrix $P^*$. The other one is negative advantages the non-optimal actions. To incorporate both sources into our analysis we characterize the dynamics of the Value Iteration update by deriving the following upper and lower bounds. For a single state $s$ with the action $a$ maximizing advantage on that state at time step $t$ and the optimal action $a^*$, the value on that state after one iteration of the algorithm can be upper bounded by:
\begin{align} \label{eq:vi_ub}
V_{t+1}(s) = V_t(s) + \adv{a}{\pi}=r^a + \gamma \sum_{i=1}^n p^a_i V_t(i) \le \gamma \sum_{i=1}^n p^a_i V_t(i),
\end{align}
where the last inequality follows from the fact that in the normal form, the reward of each action is non-positive. The same expression might be lower bounded by:
\begin{align} \label{eq:vi_lb}
V_{t+1}(s) \ge V_t(s) + \adv{a^*}{\pi}= \gamma \sum_{i=1}^n p^{a^*}_i V_t(i),
\end{align}
since action $a$ maximizes the advantage with respect to pseudo-policy implied by $V_t$ and the reward of optimal actions is $0$. Combining these two inequalities together and writing them in a matrix form, we have an inequality:
\begin{equation*}
\gamma P^* V_t \le V_{t+1} \le \gamma P_t V_t,
\end{equation*}

where $P_t$ denotes the probabilities matrix of actions chosen at step $t$, and $P^*$ is the matrix of optimal aciton probabilities. Therefore we can mix $\gamma P^* V_t$ and $\gamma P_t V_t$ with some coefficients between $0$ and $1$ to get $V_{t+1}$. Placing these coefficients in a diagonal matrix $D_t$, we obtain the following characterization of the Value Iteration dynamics:
\begin{equation} \label{eq:VI_dynamics}
V_{t+1} = \gamma P_t' V_t := \gamma [D_t P^* + (I - D_t) P_t] V_t.
\end{equation}

The value of $D_t(s,s)$ is ambiguous if $a = a^*$, in which case we set $D_t(s,s)=1$. We can place a restriction on the other values of the diagonal entries of $D_t$, \textit{i.e.} when $a \neq a^*$. Introduce the maximum advantage of non-optimal actions with respect to the optimal policy as $-\delta$:
\begin{align*}
\adv{a'}{\pi} \le -\delta < 0,  \forall a' \notin \pi^*
\end{align*}

In particular, $r^a = \adv{a}{\pi^*} \le \delta$, which implies that the value of $D_t(s,s)$ cannot be zero. We can derive the lower bound on it as follows.

\begin{gather}
\gamma \left(D_t(s,s)\sum_{i=1}^n p^{a^*}_i V_t(i) + \left(1-D_t(s,s)\right)\sum_{i=1}^n p^a_i V_t(i)\right)  = V_{t+1} (s) \le \gamma \sum_{i=1}^n p^a_i V_t(i) - \delta \\
\gamma D_t(s,s) \sum_{i=1}^n (p^{a^*}_i V_t(i) - p^a_i V_t(i)) \le -\delta \label{eq:spandelta}
\end{gather}

Consider the sum $\sum_{i=1}^n (p^{a^*}_i V_t(i) - p^a_i V_t(i))$. Inequality \ref{eq:spandelta} implies that the sum is negative. Finally, each weighted sum of $V_t$ coordinates can be bounded as follows:

\begin{gather}
- \mathrm{sp}(V_t) = \min_i V_t(i) - \max_i V_t(i) \le \sum_{i=1}^n (p^{a^*}_i V_t(i) - p^a_i V_t(i)) \\
\forall t, s:D_t(s,s) \ge \frac{\delta}{\gamma \sum_{i=1}^n (p^a_i V_t(i) - p^{a^*}_i V_t(i))} \ge \frac{\delta}{\gamma\,\spanv{V_t}}
\end{gather}

\subsection{Convergence Analysis} \label{ssec:VI_convergence}

Assumption \ref{ass:irred_and_aper_MDP} is equivalent to the fact that there exists an exponent $N$ such that all entries of the matrix $\left(P^{*}\right)^{N}$ are positive. It is known that we can choose $N \le n^2 - 2n + 2$ (see, for example \cite{holladay1958powers}). Denote the minimum of the entries of $\left(P^{*}\right)^{N}$ by $\omega$. The main theorem of this section characterizes the convergence of a standard VI algorithm:

\begin{theorem} \label{thm:sync_no_lr}
If Assumption \ref{ass:irred_and_aper_MDP} holds, the span of the value vector obtained after $N$ steps of a standard Value Iteration algorithm satisfies the following inequality:

\begin{equation*}
 \spanv{V_N} \le \gamma^N \tau\, \spanv{V_0},   
\end{equation*} 

where $\tau \in (0,1)$.
\end{theorem}

\begin{proof}

We can aggregate Equation~\ref{eq:VI_dynamics} from $t=N-1$ to $t=0$:
\begin{equation*}
V_N = \gamma^N \left( \prod_{t=N}^{1} P'_t \right) V_0
\end{equation*}

Spell out the definition of $P'_t$:
\begin{multline}
\prod_{t=N}^{1} P'_t = \prod_{t=N}^{1} \left( D_t P^* + (I - D_t) P_t \right) \ge \prod_{t=N}^1 \left( D_t P^* \right) \ge \\ \ge \prod_{t=N}^{1} \left( \frac{\delta}{\gamma\,\spanv{V_t}} P^* \right) = \frac{\delta^N}{\gamma^N\,\prod_{t=1}^N \spanv{V_t}} \left(P^*\right)^N \ge \frac{\omega \, \delta^N}{\gamma^N\,\prod_{t=1}^N \spanv{V_t}}\mathbf{1}_{n\times n}.
\end{multline}

Denote the constant $\omega \, \delta^N\! \left/ \left(\gamma^N\,\prod_{t=1}^N \spanv{V_t}\right)\right.$ by $\phi$. Note that the LHS of the inequality is a stochastic matrix, and the RHS is a matrix with all row sums equal to $n\phi$. Therefore, the normalized difference
\begin{equation*}
Q = \frac{\left(\prod_{t=N}^{1} P'_t \right) - \phi \,\mathbf{1}_{n\times n}}{1-n\phi} \ge 0
\end{equation*}
is a stochastic matrix.
Use $Q$ to express $V_N$:
\begin{equation*}
V_N = \gamma^N \left( \prod_{t=N}^{1} P'_t \right) V_0 = \gamma^N \left((1-n\phi)Q + \phi\mathbf{1}_{n\times n}\right)V_0 = \gamma^N(1-n\phi)QV_0 +n\phi\overline{V}_0,
\end{equation*}
where $\overline{V}_0$ denotes the mean of the vector $V_0$. This implies the bounds on the coordinates of $V_N$:
\begin{gather*}
\gamma^N(1-n\phi)\min(V_0) + n\phi\overline{V}_0 \le V_n \le \gamma^N(1-n\phi)\max(V_0) + n\phi\overline{V}_0 \\
\spanv{V_N} \le \gamma^N(1-n\phi)\spanv{V_0}.
\end{gather*}
Finally, we can simplify the factor $1-n\phi$:
\begin{equation*}
1-n\phi = 1-\frac{n\,\omega \, \delta^N}{\gamma^N\,\prod_{t=1}^N \spanv{V_t}} = \tau.
\end{equation*}
\end{proof}

\begin{lemma} \label{lem:adv_span}
    Given two actions $a_1, a_2$ and a policy $\pi$, we have
    \begin{equation*}
        \left|(\adv{a_1}{\pi} - \adv{a_2}{\pi}) - (r^{a_1} - r^{a_2})\right| \le \gamma\, \spanv{V^\pi}
    \end{equation*}
    if the actions are on the same state, and
    \begin{equation*}
        \left|(\adv{a_1}{\pi} - \adv{a_2}{\pi}) - (r^{a_1} - r^{a_2})\right| \le (1+\gamma) \spanv{V^\pi}
    \end{equation*}
    if they are not.
\end{lemma}

\begin{proof}
    In both cases, the quantity we want to bound can be written as
    \begin{equation*}
        (\adv{a_1}{\pi} - \adv{a_2}{\pi}) - (r^{a_1} - r^{a_2}) = \sum_{i=1}^n (c^{a_1}_i - c^{a_2}_i) V^{\pi}(i).
    \end{equation*}
    Note that the sum $\sum_{i=1}^n (c^{a_1}_i - c^{a_2}_i)$ of the coefficient differences  is equal to $0$. For the coefficients $-C \le C_i \le C, \, \forall i$ with such a property, an inequality 
    \begin{equation*}
    -C\spanv{V^\pi} \le \sum_{i=1}^n C_i V^\pi (i) \le C\spanv{V^\pi}
    \end{equation*}
    can be proved with a simple redistribution argument. For $\Delta C >0$ and indices $k$ and $l$ such that $V^\pi(k) \le V^\pi(l)$, the overall sum can be increased if $C_k$ is decreased by $\Delta C$ and $C_l$ is increased by $\Delta C$. Therefore the maximum of the sum is reached when coefficient $C$ is assigned to the $\max V^\pi$ and $-C$ assigned to the $\min V^\pi$, which implies $\max \sum_{i} C_i V^\pi(i)=C \spanv{V^\pi}$.

    The lemma then follows from the fact that the maximum absolute difference between coefficients of two actions is $\gamma$ when they are on the same state and $1+\gamma$ when they are on the different states.
\end{proof}

The following corollary uses Theorem \ref{thm:sync_no_lr} and Lemma \ref{lem:adv_span} to demonstrate the convergence of the VI algorithm.




\begin{corollary} \label{cor:disc_reward_complexity}
    A standard VI algorithm with span-based stopping criteria $H(\mathcal{I}_t): \spanv{V_t - V_{t-1}} \le \frac{\epsilon(1-\gamma)}{\gamma}$ outputs $\epsilon$-optimal policy after at most: 
\begin{equation} \label{eq:num_iter}
    \mathcal{O} \left( \frac{\log{(1/\epsilon)} + \log{(1-\gamma)}}{\log{(1/\gamma)} + \frac{\log{(1/\tau)}}{N}} \right)
\end{equation}
iterations.

\end{corollary}
\begin{proof}
First, we establish a connection between the $\spanv{V_t}$ and how close the policy $\pi_t$ is to the optimal one. For a state $s$ consider the action $a$ which maximizes the advantage with respect to $V_t$ and the action $a^*$ is the one participating in the optimal policy $\pi^*$.
Applying Lemma \ref{lem:adv_span} to actions $a$, $a^*$ and the policy $\pi^t$, we have that:

\begin{align*}
&\spanv{V^\pi} \ge \left|(\adv{a}{\pi_t} - \adv{a^*}{\pi_t}) - (r^{a} - r^{a^*})\right|=\\
&\left| -r^{a} + (\adv{a}{\pi_t} - \adv{a^*}{\pi_t}) \right| \ge |-r^a| 
\end{align*}

Therefore, the minimum possible reward of the actions in $\pi_t$ is $-\gamma \spanv{V_t}$ and the policy is  $\epsilon'$-optimal for $\epsilon' = \frac{\gamma\, \spanv{V_t}}{(1-\gamma)}$.

Second, we establish a connection between the span of $V_t$ and the stopping criterion, which is defined in terms of the span of advantages: $\spanv{V_{t+1}-V_t}=\max_{a \in \pi_t} \adv{a}{\pi_t} - \min_{a \in \pi_t} \adv{a}{\pi_t}$. Note that this span does not change when we move the pseudo-policy hyperplane $\mathcal{H}_t$ vertically. Construct an auxiliary hyperplane $\mathcal{H}_t'$ with values $V_t'$ and a corresponding pseudo-policy $\pi_t'$ parallel to $\mathcal{H}_t$. We want to choose its height such that if we consider the intersection of the hyperplane and the set of points which satisfy the coordinate constraints imposed on actions, the maximum height among the points in this intersection is 0. In other words, we want that all the maximum reward of all potential actions which lie on $\mathcal{H}_t'$ is $0$ or that $\mathcal{H}_t'$ crosses the space of possible optimal actions on the border of this space. Then, let's choose a point with the $0$ height on this hyperplane and construct an auxiliary action $a'$ in this point.

Note, that for any action $a \in \pi_t$ it's advantage with respect to $\mathcal{H}_t'$ is higher than the advantage of the optimal action, while all optimal actions lie above $\mathcal{H}_t'$, which implies that $\adv{a}{V_t'}>0\,\forall a\in \pi_t$. Then, let's apply to Lemma \ref{lem:adv_span} to actions $a$ and $a'$ and pseudo-policy $\pi_t'$:
\begin{align} \label{eq:aux_adv_bound}
|\adv{a}{\pi_t'} - r^a| \le (1+\gamma)\spanv{V_t'} \implies \adv{a}{\pi_t'} \le (1+\gamma)\spanv{V_t},
\end{align}
where the last inequality is implied by the fact that rewards are non-positive and $\spanv{V_t'} = \spanv{V_t}$. Therefore, all advantages $a \in \pi_t$ are non-negative and upper-bounded by $(1+\gamma)\spanv{V_t}$, which implies that
\begin{equation*}
\spanv{V_{t+1}-V_t}=\max_{a \in \pi_t} \adv{a}{\pi_t} - \min_{a \in \pi_t} \adv{a}{\pi_t}  
\end{equation*}
Thus, both optimality and stopping criterion depend on $\spanv{V_t}$, with stopping criteria having a larger constant.
Theorem \ref{thm:sync_no_lr} implies that after $t$ iterations the span of an value vector $V_t$ might be upper bounded by:
\begin{equation*}
    \spanv{V_t} \le \gamma^t \tau^{\lfloor t/N \rfloor} \spanv{V_0}.
\end{equation*}

Therefore, after number of iterations specified in Equation \ref{eq:num_iter} $\spanv{V_t}$ is small enough, so that stopping criterion triggers, while the policy is $\epsilon$-optimal.

    
\end{proof}

Note, that the stopping criteria $H(\mathcal{I}_t)$ does not depend on values of $N$ and $\tau$, which are unknown during the run of the algorithm. Therefore, we take an advantage of the extra convergence factor $\log{(\tau)}/N$ when its exact value is not known.

We continue with the analysis of the Value Iteration algorithm with the learning rate $\alpha < 1$. In the following, we show how this learning rate affects the contributions of the two convergence mechanisms we previously discussed: a contraction induced by the discount factor $\gamma$ and a mean reversion resulting from the mixing properties of the stochastic matrix $P_t'$. By introducing a learning rate, we create a trade-off between these two sources of convergence: as the learning rate increases, part of the contraction effect of $\gamma$ is sacrificed to enable faster information exchange between states and to strengthen the mean reversion.

One immediate result of introducing the learning rate is that now it is guaranteed that under MRP produced by optimal policy a number of updates $N_\alpha$ required to guarantee that every state affects every other state is at most $n-1$. Recall that in general case this number can be as high as $n^2- 2n +2$.

\begin{theorem} \label{thm:sync_w_lr}
Convergence of the \textbf{Synchronous algorithm with a learning rate:} If Assumption \ref{ass:irred_and_aper_MDP} hold, span of the error vector obtained after $n$ steps of synchronous Value Iteration algorithm with learning rate $\alpha \in (0,1)$ has the following property:
\begin{equation*}
    \spanv{e_{N_\alpha}} \le \gamma^{N_\alpha} \tau_\alpha \spanv{e_0},
\end{equation*}
where $\gamma^{N_\alpha}\tau_\alpha \in (0,1)$. \end{theorem}

\begin{proof}
    Proof of this theorem is similar to the proof of Theorem \ref{thm:sync_no_lr}. Full version of the proof is given in Appendix \ref{ssec:sync_w_lr_proof}. \end{proof}

Having this theorem, we can state a convergence Corollary analogous to the corollary for the standard algorithm with the identical proof.

\begin{corollary} \label{cor:disc_reward_w_l_complexity}
Then synchronous Value iteration algorithm with a learning rate $\alpha \in (0,1)$ and a stopping criteria $\spanv{V_t} < \frac{\epsilon (1-\gamma)}{\gamma (1+\gamma)}$ outputs an $\epsilon$-optimal policy after at most: 
\begin{equation} \label{eq:num_iter_wl}
    \mathcal{O} \left( \frac{\log{(1/\epsilon)} + \log{(1-\gamma)}}{\log{(1/\gamma)} + \frac{\log{(1/\tau_\alpha)}}{N_\alpha}} \right).
\end{equation}
iterations.
\end{corollary}

The Value Iteration algorithm with action filtering is discussed in Appendix \ref{ssec:action_filt}.

\section{Conclusion}

In this paper, we introduced a new geometry-based analytical framework for studying the convergence of MDP algorithms. We demonstrated how this approach can be used to obtain new results in the analysis of Policy Iteration and Value Iteration convergence.

\appendix


\bibliography{main}
\bibliographystyle{rlj}

\beginSupplementaryMaterials

\section{Additional notes}


\subsection{Idea behind the extra factor of Value Iteration Convergence in Algebraic Terms} \label{app:vi_alg_terms}

Let's write down a single update on a single state value:
\begin{align}
v_{t+1}(s) &= v^*(s) + e_{t+1}(s) =  r(s,\pi_t(s)) + \gamma\sum_{s'} P_t(s,s') v_t(s') \nonumber \\
&= r(s,\pi_t(s)) +\gamma \sum_{s'} P_t(s,s') (v^*(s') + e_{t}(s')) \nonumber \\
&= r(s,\pi_t(s)) + \gamma \sum_{s'} P^*(s,s') v^*(s') + r(s,\pi^*(s)) - r(s,\pi^*(s))  \nonumber \\
&+ \gamma \left( \sum_{s'} P_t(s,s') v^*(s') - \sum_{s'} P^*(s,s') v^*(s') \right) 
+\gamma  \sum_{s'} P_t(s,s')e_{t}(s') \nonumber \\
&= v^*(s) +\gamma  \sum_{s'} P_t(s,s')e_{t}(s') + \nonumber \\
&  \underbrace{r(s,\pi_t(s)) - r(s,\pi^*(s)) + \gamma \left( \sum_{s'} P_t(s,s') v^*(s') -\sum_{s'} P^*(s,s') v^*(s') \right) }_{\adv{a}{\pi^*}}  \nonumber \\
&\implies e_{t+1}(s) \le \gamma\sum_{s'} P_t(s,s')e_{t}(s') \nonumber
\end{align}
with the equality achieved when action $a=\pi_t(s)=\pi^*(s)$. The inequality carries the ideas of convergence of error vector span. Firstly, because non-increasing and non-decreasing properties of stochastic matrix, span will be contracting by $\gamma$ each iteration (if $\gamma < 1$). Secondly, the convergence will follow from mixing properties of matrix $P^*$ if the optimal action is chosen and from additional term $\Delta(a) \ge \delta$ otherwise.

\subsection{Action Filtering} \label{ssec:action_filt}

The Value Iteration algorithm is usually considered as an algorithm which outputs an approximate solution. In this subsection we show how under an assumption of unique optimal policy it can be used to output an exact optimal policy by applying a technique called "action filtering", similarly as in Appendix C.3 in \cite{mdp_geometry}.

In this section we want to design a filtering criteria $F(\cdot| \mathcal{I})$ such that it will guarantee, that when certain conditions are met, the action $a$ is guaranteed to be non-optimal and can be safely omitted during the subsequent iterations of the algorithm. To provide such guarantee, we need to show that the advantage of this action with respect to the optimal policy is negative. For clarity, in this section we consider a general, \textbf{not normal} MDP. Additionally, we assume that all rewards are scaled ($r^a \in [0,1] \, \forall a$ and the values are initialized by the upper bound $V_0 = \mathbf{1}*(1-\gamma)^{-1}$. 

Suppose that after $t$ iterations of the standard VI algorithm, the advantage of action $a$, $\st(a) = 1$ with respect to current pseudo-policy $V_t$ is equal to $h_t^a$, while its correct advantage with respect to optimal policy is equal to $\delta^a$. Then, if we denote error vector $e_t=V_t - V^*$ we obtain:

\begin{align} \label{eq:adv_equality}
&h_t^a = \adv{a}{V_t} = r^a + (\gamma p^a_1 - 1) V_t(1) + 
 \gamma \sum_{i=2}^n p^a_i V_t (i) = \notag \\
&= r^a + (\gamma p^a_1 - 1) (V^*(1)+ e_t(1) ) + 
 \gamma \sum_{i=2}^n p^a_i (V^*(i) + e^t(i) ) =  \notag\\ 
&= \delta^a + (\gamma (1 - \sum_{i=2}^n p^a_i)  -1) (e^t_1) + 
 \gamma \sum_{i=2}^n p^a_i  e^t_i \implies \notag \\ 
& \adv{a}{V_t} = \adv{a}{V^*} - (1-\gamma) e^t_1 + \gamma \sum_{i=2}^n p^a_i (e^t_i - e^t_1).
\end{align}
This equality ties observed quantity $h$ and the true advantage $\delta^a$, while the quantities $e^t_1$  and $(e^t_i - e^t_1)$ converge to $0$. With normalized rewards and initiation with maximum values, both of this terms can be upper bounded by $\gamma^t(1-\gamma)^{-1}$:

\begin{align} \label{eq:delta_upper_bound}
& \adv{a}{V^*} = \adv{a}{V_t} - \gamma \sum_{i=2}^n p^a_i (e_t (i) - e_t (1) ) + (1-\gamma) e_t (1) \le \\ \nonumber
& \adv{a}{V_t} + \gamma (1 - p^a_1) \spanv{e_t} + (1-\gamma)||e_t||_\infty \le \adv{a}{V_t} + (1-p^a_1)\gamma^t(1-\gamma)^{-1}, \\
\end{align}
which allows to design a filtering criteria. For an action $a$ at time step $t$ we chenk if its advantage with respect to the current pseudo-policy is smaller than the expression $(1-p^a_1)\gamma^t(1-\gamma)^{-1}$, and if this condition fulfilled, the action can be safely removed from further consideration.

It's only left to show that this condition will be eventually fulfilled for every non-optimal action, which again follows from \ref{eq:adv_equality}:
\begin{equation} \label{eq:adv_upper_bound}
h^a_t \le \delta^a + \gamma (1-p^a_1) \spanv{e_t} + (1-\gamma) ||e^t||_\infty 
\end{equation}

Combining \ref{eq:delta_upper_bound} and \ref{eq:adv_upper_bound} we have each non-optimal action $a$ is guaranteed to be filtered out once
$$ 2\gamma^t(1-p^a_1)(1-\gamma)^{-1} < -\delta$$
is true.

\section{Proofs}

\subsection{Proof of Theorem \ref{thm:J_preserve}} \label{app:J_preserve_proof}

$(1)$ Let's denote old action and policy vectors as $a^+$ and $V_+$ and new as $\bar{a}^+$ and $\bar{V}_+$. Then, 
 \begin{align*}
\adv{a}{\pi} &= a^+V_+^\pi = r^a + \sum_i c_i^a V^\pi(i)  \\
&= r^a + (\gamma - 1)V^\pi (s) + \sum_i c_i^a (V^\pi(i) - V^\pi (s))
\end{align*}
We obtained the second equality by adding and subtracting $(\gamma - 1)V^\pi$ and using the fact that $\sum_i c_i^a = \gamma - 1$. Then,

$$ (\gamma - 1)V^\pi(s) = (\gamma' - 1)\bar{V}^\pi(s)  $$

by definition of $\bar{V}^\pi(s)$ and 

$$ \sum_i c_i^a (V^\pi(i) - V^\pi (s)) = \sum_i \bar{c}_i^a (\bar{V}^\pi(i) - \bar{V}^\pi (s)), $$

since the transformation preserves coefficients $c_i$ and differences $V(i) - V(s)$ for all states except $s$ and for state $s$ where the coefficient is changing the difference is $0$. Therefore,

\begin{align*}
\adv{a}{\pi} & = r^a + \sum_i c_i^a V^\pi(i) = r^a + (\gamma - 1)V^\pi (s) + \sum_i c_i^a (V^\pi(i) - V^\pi (s)) = \\
& = \bar{r}^a + (\gamma' - 1)\bar{V}^\pi(s) + \sum_i \sum_i \bar{c}_i^a (\bar{V}^\pi(i) - \bar{V}^\pi (s)) = \\
&= \bar{r}^a  + \sum_i \bar{c}_i^a \bar{V}^\pi(i) =\bar{a}^+\bar{V}_+ = \adv{\bar{a}}{\pi}
\end{align*}

(2) For any two states $i$ and $j$

\begin{align*}
&V(i) - V(j) = (V(i) - V(s)) - (V(s) - V(j) ) \\
&=(\bar{V}(i) - \bar{V} (s) + \bar{V}(s) - \bar{V}(i) = \bar{V}(i) - \bar{V}(j)
\end{align*}

\subsection{Algebraic proof of Theorem \ref{thm:PI_2state}} \label{app:PI_2}
Choose to policies from $\mathcal{U}$ with the maximum and minimum slopes. Denote $\pi_l = \min_{\pi \in \mathcal{U}} V^\pi(1) - V^\pi(2)$ and $\pi_l = \max_{\pi \in \mathcal{U}} V^\pi(1) - V^\pi(2)$.

We need to choose the state in which the difference between the poicies increase, or, in algebraic terms, we choose state 1 if $V^{\pi_r}(1) < V^{\pi_l}(2)$ and choose state 2 otherwise. Without loss of generality, let's assume it is state 2, $V^{\pi_r}(2) > V^{\pi_l}(2)$. We denote the actions on state as $b \in \pi_r$ and $c \in \pi_l$ (same as on the Figure \ref{fig:pi_proof}). Then, we construct two auxiliary actions $e_l$ and $e_r$, which are located in the same places where $\pi_l$ and $\pi_r$ cross the state 2 value line or, in other words, $e_1$ and $e_2$ are self-loop actions on state 2, $e_{l+} = ( (1-\gamma)V^{\pi_l}, 0, \gamma-1)$ and $e_{r+} = ( (1-\gamma)V^{\pi_r}, 0, \gamma-1)$.

Then, for any policy $\pi \in U$ the following inequality holds:
\begin{equation} \label{eq:inefficiency_zone}
\adv{c}{\pi} \le \adv{e_l}{\pi} < \adv{e_r}{\pi} \le \adv{e}{\pi}.
\end{equation}

To prove the first inequality note then both $c$ and $e_l$ lie on the $\mathcal{H}^{\pi_l}$, which implies that:
\begin{equation*}
\adv{c}{\pi_l} - \adv{e_l}{\pi_l}=0 \implies r^{e_l} = r^c + \gamma p_1^c (V^{\pi_l}(1) - V^{\pi_l}(2))
\end{equation*}

Then, for any policy with $\pi$ with values $V^\pi$:
\begin{align*}
\adv{c}{\pi} &= r^c + \gamma p_1^c V^\pi(1) + (\gamma p_2^c - 1)V^\pi(2)\\
&=[r^c+ \gamma p_1^c (V^\pi(1) - V^\pi(2)] + 0\cdot V^\pi(1) + (\gamma - 1) V^\pi(2)\\
&\le r^c + \gamma p_1^c (V^{\pi_l}(1) - V^{\pi_l}(2)) + 0\cdot V^\pi(1) + (\gamma - 1) V^\pi(2)\\
&=r^{e_l} +  0\cdot V^\pi(1) + (\gamma - 1) V^\pi(2) = \adv{e_l}{\pi}
\end{align*}

The second inequality in \ref{eq:inefficiency_zone} follows from the fact that $e_l$ and $e_r$ have the same coefficients, but $e_r$ has strictly higher reward. The third inequality can be proven the same way as the first one.

Therefore, for any policy $\pi$ action $b$ has higher advantage than action $c$, which implies that $c$ in not efficient on $\mathcal{U}$ and cannot be chosen after one update.

\subsection{Proof of Value Iteration Convergence with Learning rate}\label{ssec:sync_w_lr_proof}

With learning rate introduced one iteration of the algorithm is:
\begin{equation*}
    V_{t+1}(s) \leftarrow V_{t}(s)(1-\alpha) + \alpha \max_a \left[ r(s,a) +\gamma \sum_{s'} P(s'|s,a) V_{t}(s') \right]
\end{equation*}

and value transition dynamics equality similar to \ref{eq:VI_dynamics} becomes:

\begin{align} \label{eq:error_transition_lr}
    V_{t+1} = \big((1-\alpha)I + \gamma \alpha [D_t P^* + (I-D_t) P_t]\big) V_t = \gamma P_{t,\alpha}' V_t,
\end{align}

Note that $P_{t,\alpha}'$ is a stochastic matrix only in a case when $\gamma = 1$, but it is sufficiently close to it since we assume that $\gamma$ is almost $1$. Additionally, $N_\alpha \le n-1$ to guarantee that the matrix $P_{t,\alpha}'^{N_\alpha}$ is positive, since elements on the main diagonal coming from $(1-\alpha)/\gamma I$ influence error dynamics the similar way as having a loop in every state, thus every state will be affected by every other state in at most $n-1$. 

Consequently, the minimum values of $\delta'$ needs to be updated, now we have that $P_{t,\alpha}'(s,s')\ge\alpha \delta'$ for states $s,s':s\ne s'$ and $P_{t,\alpha}'(s,s) \ge (1-\alpha)\gamma $. Let's define $\delta'_\alpha$ as a minimum of these two quantities. Thus, an expression of the of an error associated with state $s$ after $N$ iterations becomes:

\begin{align*}
    V_N(s) &= \gamma^N\sum_{s'\in \mathcal{S}}\lambda_{s'} V_0(s') = \gamma^N\sum_{s'\in \mathcal{S}}\delta^{'N}_\alpha V_0 (s') + (\lambda_{s'}-\delta^{'N}_\alpha) V_0(s') = \\ 
    &= \gamma^Nn \delta^{'N}_\alpha \overline{V}_0 + 
\gamma^N\sum_{s'\in \mathcal{S}} (\lambda_{s'}-\delta^{'N}_\alpha) V_0(s'),\\
\end{align*} 

Note, that now the sum of coefficients $\lambda_{s'}$ is not $1$, but $[(1-\alpha)/\gamma + \alpha]^N$. This gives us a final convergence rate of:

\begin{equation*}
    \spanv{V_N} \le \gamma^N ([(1-\alpha)/\gamma + \alpha]^N-n\delta_\alpha^{'N}) \spanv{V_0}
\end{equation*}

Defining $([(1-\alpha)/\gamma + \alpha]^N-n\delta_\alpha^{'N})$ as $\tau_\alpha$ we have the claimed result.

\end{document}